\pdfoutput=1 
\documentclass[runningheads]{llncs}
\usepackage{graphicx}
\usepackage{amsmath,amssymb} % define this before the line numbering.
\usepackage{color}
\usepackage[colorlinks]{hyperref}

\usepackage{multirow}
\usepackage{xspace}
\usepackage{xcolor}
\usepackage{algorithm}
\usepackage[noend]{algpseudocode}
\usepackage{xfrac}
\usepackage{amsmath}
\usepackage{enumitem}
\DeclareMathOperator{\vv}{vec}
\newcommand{\bigoh}{\mathcal{O}}

\begin{document}
\title{Second-order Democratic Aggregation} 
% Replace with your title

\titlerunning{Second-order Democratic Aggregation}
% Replace with a meaningful short version of your title
%
\author{Tsung-Yu Lin\inst{1} \and 
	Subhransu Maji\inst{1} \and
	 Piotr Koniusz\inst{2}}
\institute{College of Information and Computer Sciences\\
	University of Massachusetts Amherst\\
	\email{ \{tsungyulin,smaji\}@cs.umass.edu}
	\and
	Data61/CSIRO,
	Australian National University \\
	\email{piotr.koniusz@data61.csiro.au}
}

%
%Please write out author names in full in the paper, i.e. full given and family names. 
%If any authors have names that can be parsed into FirstName LastName in multiple ways, please include the correct parsing, in a comment to the volume editors:
%\index{Lastnames, Firstnames}
%(Do not uncomment it, because you may introduce extra index items if you do that, we will use scripts for introducing index entries...)
\authorrunning{Tsung-Yu Lin, Subhransu Maji and Piotr Koniusz}
% Replace with shorter version of the author list. If there are more authors than fits a line, please use A. Author et al.
%

\def\eg{\emph{e.g.}}
\def\ie{\emph{i.e.}}
\def\etc{\emph{etc.}}

\maketitle              % typeset the header of the contribution
\begin{abstract}
Aggregated second-order features extracted from deep convolutional networks have been shown to be effective for texture generation, fine-grained recognition, material classification, and scene understanding.
In this paper, we study a class of orderless aggregation functions designed to minimize \emph{interference} or equalize \emph{contributions} in the context of second-order features and we show that they can be computed just as efficiently as their first-order counterparts and they have favorable properties over aggregation by summation.
Another line of work has shown that matrix power normalization after aggregation can significantly improve the generalization of second-order representations.
We show that matrix power normalization implicitly equalizes contributions during aggregation thus establishing a connection between matrix normalization techniques and prior work on minimizing interference.
Based on the analysis we present $\gamma$-democratic aggregators that interpolate between sum ($\gamma$=1) and democratic pooling ($\gamma$=0) outperforming both on several classification tasks.
Moreover, unlike power normalization, the $\gamma$-democratic aggregations can be computed in a low dimensional space by sketching that allows the use of 
very high-dimensional second-order features. This results in a state-of-the-art performance on several datasets.
\end{abstract}
\def\point{P}
\def\matrix{P}
\def\npoints{N}
\def\nshapes{S}
\def\nbasis{B}
\def\depth{D}
\def\scale{k}
\def\encoding{\mathbf{z}}
\newcommand{\para }[1]{\medskip \noindent {\underline {\bf #1}}}
\def\mrtnet{MRTNet\xspace}
\def\mrvae{MR-VAE\xspace}
\def\kdtree{kd-tree\xspace}

\section{Introduction}
\label{sec:intro}
Second-order statistics have been demonstrated to improve performance of classification on images of objects, scenes and textures as well as fine-grained problems, action classification and tracking \cite{tuzel_rc,porikli2006tracker,locallog-euclidean,guo2013action,carreira_secord,me_tensor,bilinear_finegrained,lin2017improved}. In the simplest form, such statistics are obtained by taking the outer product of some feature vectors and aggregating them over some region of interest which results in an auto-correlation \cite{carreira_secord,me_tensor_tech_rep} or covariance matrix \cite{tuzel_rc}. Such a second-order image descriptor is then passed as a feature to train a SVM, \etc~ Several recent works obtained an increase in accuracy after switching from the first- to second-order statistics \cite{me_tensor,me_tensor_tech_rep,bilinear_finegrained,lin_bcnn_pami,lin2017improved,yu_stat,koniusz2016tensor,koniusz2018deeper}. Further improvements were obtained by considering the impact of spectrum of such statistics on aggregation into the final representations \cite{me_tensor,me_tensor_tech_rep,murray2017interferences,secord_peihua_li,lin2017improved,koniusz2016tensor,koniusz2018deeper}. For instance, analysis conducted in \cite{me_tensor,me_tensor_tech_rep} concluded that decorrelating feature vectors from an image via the matrix power normalization has a positive impact on classification due to the signal whitening properties which prevent so-called {\em bursts} of features \cite{jegou_bursts}. However, evaluating the power of matrix is a costly procedure with complexity $\bigoh(d^\omega)$, where $2\!<\!\omega\!<\!2.376$ concerns the complexity of SVD. In recent CNN approaches \cite{secord_peihua_li,lin2017improved,koniusz2018deeper} which perform end-to-end learning, the complexity becomes a prohibitive factor for typical $d\geq 1024$ due to a costly backpropagation step which involves SVD or solving a Lyapunov equation \cite{lin2017improved} in every iteration of the CNN fine-tuning process; thus adding several hours of computations to training. However, another line of aggregation mechanisms aim to reweight the first-order feature vectors prior to their aggregation \cite{murray2017interferences} in order to balance their contributions to the final image descriptor. Such a reweighting scheme, called a democratic aggregation \cite{jegou_democratic,murray2017interferences}, is solved very efficiently by a modified Sinkhorn algorithm \cite{Knight_Sinkhorn}.%  with only a few iterations. 

In this paper, we study democratic aggregation in the context of second-order feature descriptors and show that this feature descriptor has favorable properties when combined with the democratic aggregator which was applied originally to the first-order descriptors. We take a closer look at the relation between the reweighted representations and the matrix power normalization in terms of the variance of feature contributions. In addition, we propose a $\gamma$-democratic aggregation scheme which generalizes democratic aggregation and allows to interpolate between sum pooling and democratic pooling. We show that our formulation can be solved via the Sinkhorn algorithm as efficiently as approach \cite{murray2017interferences} while resulting in a performance comparable to the matrix power normalization. Computationally, our approach involves Sinkhorn iterations, which requires matrix-vector multiplications, and is faster by an order of magnitude even when compared to approximate matrix power normalization via the Newton's method, which involves matrix-matrix operations \cite{lin2017improved}. Unlike the power matrix normalization, our $\gamma$-democratic aggregation can be performed via sketching \cite{Pham_sketch,Gao_2016_CVPR}  enabling the use of high-dimensional feature vectors.

To summarize, our contributions are: (i) we propose a new second-order  $\gamma$-democratic aggregation, (ii) we obtain reweighting factors via the Sinkhorn algorithm which enjoys an order of magnitude speedup over the fast matrix power normalization via Newton's iterations while it achieves comparable results, (iii) we provide theoretical bounds on feature contributions in relation to the matrix power normalization, (iv) we present state-of-the-art results on several datasets by applying democratic aggregation of second-order representations with sketching.

\section{Related work}
\label{sec:related}
Mechanisms of aggregating first- and second-order features have been extensively studied in the context of image retrieval, texture and object recognition \cite{perronnin_fisher,perronnin_fisherimpr,sanchez_fisherpract,jegou_vlad,philippe_vlat,porikli2006tracker,tuzel2008detection,carreira_secord,me_tensor}. In what follows, we first describe shallow approaches and non-Euclidean aggregation schemes followed by the %state-of-the-art end-to-end 
CNN-based approaches.

%\smallskip
%\vspace{0.1cm}
%\noindent{\textbf{Shallow Approaches. }}
\paragraph{\textbf{Shallow Approaches. }}
Early approaches to aggregating second-order statistics include Region Covariance Descriptors \cite{porikli2006tracker,tuzel2008detection}, Fisher Vector Encoding~\cite{perronnin_fisher,perronnin_fisherimpr,sanchez_fisherpract} and Vector of Locally Aggregated Tensors~\cite{philippe_vlat}, to name but a few of approaches.

Region Covariance Descriptors capture co-occurrences of luminance, first- and second-order partial derivatives of images \cite{porikli2006tracker,tuzel2008detection} and, in some cases, even binary patterns \cite{elbcm_brod}. The main principle of these approaches is to aggregate the co-occurrences of some feature vectors into a matrix which represents an image. %Such descriptors have been widely used in tracking and pedestrian detection \cite{porikli2006tracker,tuzel2008detection}, semantic segmentation and object classification \cite{carreira_secord,me_tensor}.

Fisher Vector Encoding \cite{perronnin_fisher} precomputes a visual vocabulary by clustering over a set of feature vectors and captures the element-wise squared difference between each feature vector and its nearest cluster center.
%Fisher Vector Encoding \cite{perronnin_fisher} employs clustering of feature vectors %from the training set 
%to build a visual vocabulary. Subsequently, each feature vector from an image is assigned to its nearest cluster. The second-order information is captured by the element-wise squared difference of each feature vector and the cluster center. 
Subsequently, the re-normalization of the captured statistics with respect to the cluster variance and the sum aggregation are performed. Furthermore, extension \cite{perronnin_fisherimpr} proposes to apply the element-wise square root to the aggregated statistics which improves the classification results. 
Vector of Locally Aggregated Tensors %builds on so-called VLAD image descriptor \cite{jegou_vlad} and 
extends Fisher Vector Encoding to second-order off-diagonal feature interactions. % captured by the outer-product of feature vectors centered with respect to the cluster centers and re-normalized with respect to the cluster covariance.

%Region Covariance Descriptors are somewhat related to our method as they aggregate second-order statistics into symmetric positive (semi)definite matrices for which their non-Euclidean geometry is used for better classification results.

%\smallskip
%\vspace{0.1cm}
\paragraph{\textbf{Non-Euclidean Distances. }}
%\noindent{\textbf{Non-Euclidean Distances. }}
%
To take the full advantage of statistics captured by the scatter matrices, several works employ non-Euclidean distances. For positive definite matrices, geodesic distances (or their approximations) known from the Riemannian geometry are used \cite{PEN06,bhatia_pdm,arsigny2006log}. Power-Euclidean distance \cite{dryden_powereuclid} extends to semidefinite positive matrices. 
Distances such as Affine-Invariant Riemannian Metric \cite{PEN06,bhatia_pdm}, KL-Divergence Metric \cite{wang_jeffreys}, Jensen-Bregman LogDet Divergence \cite{anoop_logdet} and Log-Euclidean distance \cite{arsigny2006log} are frequently used for comparing scatter matrices resulting from aggregation of second-order statistics. However, the above distances are notoriously difficult to backpropagate through for end-to-end learning and often computationally prohibitive \cite{koniusz2018museum}.

%\smallskip
%\vspace{0.1cm}
%\noindent{\textbf{Pooling Normalizations. }}
\paragraph{\textbf{Pooling Normalizations. }}
Both first- and second-order aggregation methods often employ normalizations of pooled feature vectors. The early works on image retrieval apply the square root \cite{jegou_bursts} to aggregated feature vectors to limit the impact of frequently occurring features and boost the impact of infrequent and highly informative ones (so-called notion of feature {\em bursts}). The roots of this approach in computer vision can be traced back to so-called generalized histogram of intersection kernel \cite{boughorbel_intersect}. For second-order approaches, similar strategy is used by Fisher Vector Encoding \cite{perronnin_fisherimpr}. The notion of {\em bursts} is further studied in the context of Bags-of-Words approach as well scatter matrices and tensors for which their spectra are power normalized \cite{me_tensor_tech_rep,me_tensor,koniusz2016tensor} (so-called Eigenvalue Power Normalization or EPN for short). However, the square complexity of scatter matrices w.r.t. length of feature vectors deems them somewhat impractical in classification. A recent study \cite{jegou_democratic,murray2017interferences} shows how to exploit second-order image-wise statistics and reweight sets of feature vectors per image at the aggregation time to obtain an informative first-order representation. So-called Democratic Aggregation (DA) and Generalized Max-Pooling (GMP) strategies are proposed whose goal is to reweight feature vectors per image prior to the sum aggregation so that interference between frequent and infrequent feature vectors is minimized. Strategies such as EPN (Matrix Power Normalization, MPN for short, is a special case of EPN), DA and GMP can be seen as ways of equalizing contributions of feature vectors into the final image descriptor and they are closely related to Zero-phase Component Analysis (ZCA) whose role is to whiten the signal representation.

%\smallskip
%\vspace{0.1cm}
%\noindent{\textbf{Pooling and Aggregation in CNNs. }}
\paragraph{\textbf{Pooling and Aggregation in CNNs. }}
The early image retrieval and recognition CNN-based approaches aggregate first-order statistics extracted from the CNN maps \eg, \cite{orderless_pooling,deep_aggreg,netvlad}. In \cite{orderless_pooling}, multiple feture vectors are aggregated over multiple image regions. In \cite{deep_aggreg}, feature vectors are aggregated for retrieval. In \cite{netvlad}, so-called VLAD descriptor is extended to allow end-to-end training.

More recent approaches form co-occurrence patterns from CNN feature vectors similar in spirit to Region Covariance Descriptors. In \cite{bilinear_finegrained}, the authors combine two CNN streams of feature vectors via outer product and demonstrate that such a setup is robust for the task of the fine-grained image recognition. A recent approach \cite{deep_cooc} extracts feature vectors at two separate locations in feature maps and performs an outer product to form a CNN co-occurrence layer. 
%
%Approach \cite{gen_orderless} visualizes the role of normalizations and confirms the link between whitening and the saliency of features known from earlier works \cite{jegou_bursts,murray2017interferences}.

Furthermore, a number of recent approaches are dedicated to performing backpropagation on the spectrum-normalized scatter matrices \cite{sminchisescu_matrix,vangol_riem_net,secord_peihua_li,lin2017improved,koniusz2018deeper}. In \cite{sminchisescu_matrix}, the authors employ the backpropagation via the SVD of matrix to implement the Log-Euclidean distance in end-to-end fashion. In \cite{secord_peihua_li}, the authors extend Eigenvalue Power Normalization \cite{me_tensor} to an end-to-end learning scenario which also requires to backpropagate via the SVD of matrix. Concurrently, approach \cite{lin2017improved} suggests to perform Matrix Power Normalization via the Newton's method and backpropagate w.r.t. the square root of matrix by solving a Lyapunov equation for greater numerical stability. 
An approach \cite{secord_mathieu} phrases the matrix normalization as the problem of robust covariance estimation. Lastly, compact bilinear pooling \cite{Gao_2016_CVPR} uses so-called tensor sketching \cite{Pham_sketch}. Where indicated, we also make use of tensor sketching in our work.

%We note that even in the Democratic Aggregation and Generalized Max-Pooling \cite{jegou_democratic,murray2017interferences}, 
There has been no connection made between reweighting feature vectors and its impact on the spectrum of the corresponding scatter matrix. Our work closely related to the approaches \cite{jegou_democratic,murray2017interferences}, however, %we perform end-to-end-learning{\color{red}{?}} and 
introduce a mechanisms of limiting the interference in the context of second-order features. We demonstrate their superiority over the first-order inference approaches \cite{jegou_democratic,murray2017interferences} and show  that we can obtain results comparable to the matrix square root aggregation \cite{lin2017improved} with much lower computational complexity at the training and testing stages.

\section{Method}
\label{sec:method}
Given a sequence of features ${\cal X} = (\mathbf{x}_1, \mathbf{x}_2, \ldots,\mathbf{x}_n)$, where $\mathbf{x}_i \in \mathbb{R}^d$, we are interested in a class of functions that compute an \emph{orderless aggregation} of the sequence to obtain a global descriptor $ \xi({\cal X})$. 
If the descriptor is orderless, it implies that any permutation of features does not effect the global descriptor. 
A common approach is to encode each feature using a non-linear function $\phi({\mathbf x})$ before aggregation via a simple symmetric function such as sum or max. For example, the global descriptor using sum pooling can be written as:
\begin{equation}
	\xi({\cal X}) = \sum_{\mathbf{x} \in {\cal X}} \phi({\mathbf x}).
\end{equation}
In this work, we investigate outer-product encoders, \ie~ $\phi(\mathbf{x}) = \vv(\mathbf{x}\mathbf{x}^T)$, where $\mathbf{x}^T$ denotes the transpose %of a matrix 
and $\vv(\cdot)$ is the vectorization operator. Thus, if $\mathbf{x}$ is $d$ dimensional then  $\phi(\mathbf{x})$ is $d^2$ dimensional. 

\subsection{Democratic aggregation}
The democratic aggregation approach was proposed in~\cite{murray2017interferences} to minimize interference or equalize contributions of each element in the sequence. The contribution of a feature is measured as the similarity of the feature to the overall descriptor. In the case of sum pooling, the contribution $C(\mathbf{x})$ of a feature $\mathbf{x}$ is given by:
\begin{equation}
C(\mathbf{x}) = \phi(\mathbf{x})^T \sum_{\mathbf{x}'\in {\cal X}} \phi(\mathbf{x}').
\end{equation}
For sum pooling, the contributions $C(\mathbf{x})$ may not be equal for all features $\mathbf{x}$.
In particular, the contribution is affected by both the norm and frequency of the feature. 
Democratic aggregation is a scheme that weights each feature by a scalar $\alpha(\mathbf{x})$ that depends on both $\mathbf{x}$ and the overall set of features in ${\cal X}$ such that the weighted aggregation $\xi({\cal X})$ satisfies:
\begin{equation}
\alpha(\mathbf{x}) \phi(\mathbf{x})^T \xi({\cal X}) = 
\alpha(\mathbf{x}) \phi(\mathbf{x})^T \sum_{\mathbf{x}'\in {\cal X}} \alpha(\mathbf{x}')\phi(\mathbf{x}') = C,~\forall \mathbf{x} \in {\cal X},
\end{equation}
under the constraint that $\forall \mathbf{x} \in {\cal X}$, $\alpha(\mathbf{x}) > 0$.
The above equation only depends on the dot product between the elements since: 
\begin{equation}
\alpha(\mathbf{x}) \sum_{\mathbf{x}'\in {\cal X}} \alpha(\mathbf{x}')\phi(\mathbf{x})^T \phi(\mathbf{x}') =
\alpha(\mathbf{x}) \sum_{\mathbf{x}'\in {\cal X}} \alpha(\mathbf{x}')k(\mathbf{x},\mathbf{x}'), 
\end{equation}
where $k(\mathbf{x},\mathbf{x}')$ denotes the dot product between the two vectors $\phi(\mathbf{x})$ and $\phi(\mathbf{x}')$.
Following the notation in~\cite{murray2017interferences}, if we denote $\mathbf{K}_{\cal X}$ to be the kernel matrix of the set ${\cal X}$, the above constraint is equivalent to finding a vector of weights $\boldsymbol{\alpha}$ such that:
\begin{equation}
\texttt{diag}(\boldsymbol{\alpha})\mathbf{K}\texttt{diag}(\boldsymbol{\alpha})\mathbf{1}_n = C\mathbf{1}_n,
\end{equation}
where $\texttt{diag}$ is the diagonalization operator and $\mathbf{1}_n$ is an $n$ dimensional vector of ones. In practice, the aggregated features $\xi({\cal X})$ are $\ell_2$ normalized hence the constant $C$ does not matter and can be set to 1.

The authors~\cite{murray2017interferences} noted that the above equation can be efficiently solved by a dampened Sinkhorn algorithm~\cite{Knight_Sinkhorn}. The algorithm returns a unique solution as long as certain conditions are met, namely the entries in $\mathbf{K}$ are non-negative and the matrix is not fully decomposable.
In practice, these conditions are not satisfied since the dot product between two features can be negative.
A solution proposed in~\cite{murray2017interferences} is to compute $\boldsymbol{\alpha}$ by setting the negative entries in $\mathbf{K}$ to zero.

For completeness, the dampened Sinkhorn algorithm is included in Algorithm~\ref{alg:sinkhorn}. Given $n$ features of $d$ dimensions, computing the kernel matrix takes $\bigoh(n^2d)$, whereas each Sinkhorn iteration takes $\bigoh(n^2)$ time.
In practice, 10 iterations are sufficient to find a good solution. The damping factor $\tau=0.5$ is typically used. This slows the convergence rate but avoids oscillations and other numerical issues associated with the undampened version ($\tau=1$).

\subsubsection{$\gamma$-democratic aggregation.} We propose a parametrized family of democratic aggregation functions that interpolate between sum pooling and fully democratic pooling. Given a parameter $0 \leq \gamma \leq 1$, the $\gamma$-democratic aggregation is obtained by solving for a vector of weights $\boldsymbol{\alpha}$ such that:
\begin{equation}
\texttt{diag}(\boldsymbol{\alpha})\mathbf{K}\texttt{diag}(\boldsymbol{\alpha})\mathbf{1}_n = (\mathbf{K}\mathbf{1}_n)^\gamma.
\end{equation}
When $\gamma=0$, this corresponds to the democratic aggregation, and when $\gamma=1$, this corresponds to sum aggregation since $\boldsymbol{\alpha} = \mathbf{1}_n$ satisfies the above equation.
The above equation can be solved by modifying the update rule for computing $\sigma$ in the Sinkhorn iterations to:
\begin{equation}
\sigma = \texttt{diag}(\boldsymbol{\alpha})\mathbf{K}\texttt{diag}
(\boldsymbol{\alpha})\mathbf{1}_n/(\mathbf{K}\mathbf{1}_n)^\gamma, 
\end{equation}
in Algorithm~\ref{alg:sinkhorn}, where $/$ denotes element-wise division. 
Thus, the solution can be equally efficient for any value of $\gamma$. 
Intermediate values of $\gamma$ allow the contributions $C(\mathbf{x})$ of each feature $\mathbf{x}$ within the set to vary and, in our experiments, we find this can lead to better results than the extremes (\ie, $\gamma=1$).

\begin{algorithm}[t]
\caption{Dampened Sinkhorn Algorithm}\label{alg:sinkhorn}
\begin{algorithmic}[1]
\Procedure{Sinkhorn}{$\mathbf{K}, \tau, \text{T}$}
\State $\boldsymbol{\alpha}\gets \mathbf{1}_n$
\For{$t=1$ to T}
\State $\boldsymbol{\sigma} = \texttt{diag}(\boldsymbol{\alpha})\mathbf{K}\texttt{diag}(\boldsymbol{\alpha})\mathbf{1}_n$ 
\State $
\boldsymbol{\alpha} \gets \boldsymbol{\alpha}/\boldsymbol{\sigma}^\tau
$
\EndFor\label{euclidendwhile}
\State \textbf{return} $\boldsymbol{\alpha}$
\EndProcedure
\end{algorithmic}
\end{algorithm}

\subsubsection{Second-order democratic aggregation.}
In practice, features extracted using deep ConvNets can be high-dimensional. 
For example, an input image $I$ is passed through layers of a ConvNet to obtain a feature map  $\mathbf{\Phi}(I)$ of size $W \times H \times D$. 
Here $d=D$ corresponds to the number of filters in the convolutional layer and $n=W\times H$ corresponds to the spatial resolution of the feature.
For state-of-the-art ConvNets from which features are typically extracted, the values of $n$ and $d$ are comparable and in the range of a few hundred to a thousand. 
Thus, explicitly realizing the outer products can be expensive. 
Below we show several properties of democratic aggregation with outer-product encoders. 
Some of these properties allow aggregation in a computationally and memory efficient manner. 

\begin{proposition} 
\label{prop:uniqueness}
For outer-product encoders, the solution to the $\gamma$-democratic kernels exists for all values of $\gamma$ as long as $||\mathbf{x}|| > 0, ~\forall \mathbf{x} \in {\cal X}$.
\end{proposition}

\begin{proof}
%\textit{Proof:} 
For the outer-product encoder we have:
%\begin{eqnarray*}
%k(\mathbf{x},\mathbf{x}') &=& \phi(\mathbf{x})^T\phi(\mathbf{x}') \\
%%&=& \vv(\mathbf{x}\mathbf{x}^T)^T\vv(\mathbf{x'}\mathbf{x}'^T) \\
%%&=& \texttt{Trace}((\mathbf{x}'\mathbf{x}'^T)^T \mathbf{x}\mathbf{x}^T) \\
%%&=& \texttt{Trace}(\mathbf{x}'^T \mathbf{x}\mathbf{x}^T\mathbf{x}') \\
%&=& (\mathbf{x}^T \mathbf{x}')^2.
%\end{eqnarray*}

\begin{eqnarray*}
k(\mathbf{x},\mathbf{x}') = \phi(\mathbf{x})^T\phi(\mathbf{x}')
= \vv(\mathbf{x}\mathbf{x}^T)^T\vv(\mathbf{x'}\mathbf{x}'^T) 
%&=& \texttt{Trace}((\mathbf{x}'\mathbf{x}'^T)^T \mathbf{x}\mathbf{x}^T) \\
%&=& \texttt{Trace}(\mathbf{x}'^T \mathbf{x}\mathbf{x}^T\mathbf{x}') \\
= (\mathbf{x}^T \mathbf{x}')^2 \geq 0.
\end{eqnarray*}

Thus, all the entries of the kernel matrix are non-negative and the kernel matrix is strictly positive definite when $||\mathbf{x}|| > 0,  ~\forall \mathbf{x} \in {\cal X} $. This is a sufficient condition for the solution to exist \cite{Knight_Sinkhorn}. %(Theorem ???). 
Note that the kernel matrix of the outer product encoders is positive even when $\mathbf{x}^T \mathbf{x}' < 0$.
\end{proof}

\begin{proposition}
\label{prop:comlexity}
For outer-product encoders, the solution $\boldsymbol{\alpha}$ to the $\gamma$-democratic kernels can be computed in $\bigoh(n^2d)$ time and $\bigoh(n^2 + nd)$ space.
\end{proposition}
\begin{proof}
%\textit{Proof:} 
The running time of the Sinkhorn algorithm is dominated by the time to compute the kernel matrix $\mathbf{K}$. Naively computing the kernel matrix for $d^2$ dimensional features would take $\bigoh(n^2d^2)$ time and $\bigoh(n^2 + nd^2)$ space. However, since the kernel entries of the outer products are just the square of the kernel entries of the features before the encoding step, one can compute the kernel $\mathbf{K}$ by simply squaring the kernel of the raw features, which can be computed in $\bigoh(n^2d)$ time and $\bigoh(n^2+nd)$ space. Thus the weights $\boldsymbol{\alpha}$ for the second-order features can also be computed in $\bigoh(n^2d)$ time and $\bigoh(n^2 + nd)$ space.
\end{proof}

\begin{proposition}
\label{prop:sketch}
For outer-product encoders, $\gamma$-democratic aggregation $\xi({\cal X})$ can be computed with low-memory overhead using Tensor Sketching.
\end{proposition}
\begin{proof}
% \textit{Proof:} 
Let $\theta$ be a low-dimensional embedding that approximates the inner product between two outer-products, \ie, 
\begin{equation}
\theta(\mathbf{x})^T\theta(\mathbf{x'}) \sim  \vv(\mathbf{x}\mathbf{x}^T)^T\vv(\mathbf{x'}\mathbf{x}'^T), 
\end{equation}
and $\theta(\mathbf{x}) \in \mathbb{R}^{k}$ with $k << d^2$. Since the $\gamma$-democratic aggregation of ${\cal X}$ is a linear combination of the outer-products, the overall feature $\xi({\cal X})$ can be written as:
\begin{equation}
\xi({\cal X}) = \sum_{\mathbf{x} \in {\cal X}}  \alpha(\mathbf{x}) \mathbf{x}\mathbf{x}^T \sim  \sum_{\mathbf{x} \in {\cal X}} \alpha(\mathbf{x}) \theta(\mathbf{x}).
\end{equation}
\end{proof}
Thus, instead of realizing the overall feature $\xi({\cal X})$ of size $d^2$, one can use the embedding $\theta$ to obtain a feature of size $k$ as a democratic aggregation of the approximate outer-products. One example of an approximate outer-product embedding is the Tensor Sketching (TS) approach of Pham and Pagh~\cite{Pham_sketch}. 
Tensor sketching has been used to approximate second-order sum pooling~\cite{Gao_2016_CVPR} resulting in an order-of-magnitude savings in space at a marginal loss in performance on classification tasks.
Our experiments show that sketching also performs well in the context of democratic aggregation.

\subsection{Spectral normalization of second-order representations}
A different line of work~\cite{carreira_secord,lin2017improved,secord_peihua_li,secord_mathieu} has investigated matrix functions to normalize the second-order representations obtained by sum pooling. For example, the improved bilinear pooling \cite{lin2017improved} and second-order approaches \cite{me_tensor_tech_rep,me_tensor,koniusz2018deeper} construct a global representation by sum pooling of outer-products: 
\begin{equation}
\mathbf{A} = \sum_{\mathbf{x} \in {\cal X}} \mathbf{x}\mathbf{x}^T.
\end{equation}
The matrix $\mathbf{A}$ is subsequently normalized using matrix power function $\mathbf{A}^p$ with $0 < p < 1$. When $p=1/2$, this corresponds to the matrix square-root which is defined as matrix $\mathbf{Z}$ such that $\mathbf{Z}\mathbf{Z}=\mathbf{A}$. 
Matrix function can be computed using the Singular Value Decomposition (SVD).
Given matrix $\mathbf{A}$ with a SVD given by $\mathbf{A} = \mathbf{U}\Lambda\mathbf{U}^T$, where the matrix $\Lambda = \texttt{diag}(\lambda_1,\lambda_2,...,\lambda_d)$, with $\lambda_{i} \geq \lambda_{i+1}$, the matrix function $f$ can be written as $\mathbf{Z} = f(\mathbf{A}) = \mathbf{U}g(\Lambda)\mathbf{U}^T$, where $g$ is applied to the elements in the diagonal of $\Lambda$. Thus, the matrix power can be computed as $\mathbf{A}^p = \mathbf{U}\Lambda^p\mathbf{U}^T = \mathbf{U}\texttt{diag}(\lambda_1^p,\lambda_2^p,...,\lambda_d^p)\mathbf{U}^T$. %Thus these techniques are also called as
Such spectral normalization techniques  scale the spectrum of the matrix $\mathbf{A}$. The following establishes a connection between the spectral normalization techniques and democratic pooling.

Let $\mathbf{\hat{A}}^p$ be the $\ell_2$ normalized version of $\mathbf{A}^p$ and $r_{\max}$ and $r_{\min}$ be the maximum and minimum squared radii of the data $\mathbf{x} \in {\cal X}$ defined as:
\begin{equation}
r_{\max}=\max_{~\mathbf{x}~\in {\cal X}}||\mathbf{x}||^2,~r_{\min}=\min_{~\mathbf{x}~\in {\cal X}}||\mathbf{x}||^2.
\end{equation}

As earlier, let $C(\mathbf{x})$ be the contribution of the vector $\mathbf{x}$ to the the aggregated representation defined as:
\begin{equation}
\label{eq:contributions}
C(\mathbf{x}) = \vv(\mathbf{x}\mathbf{x}^T)^T\vv(\mathbf{\hat{A}}^p).
\end{equation}

\begin{proposition}
\label{prop:four_props}
The following properties hold true:
\begin{enumerate}[itemsep=1.5mm]
\item The $\ell_2$ norm of $\vv(\mathbf{A}^p)$ is $\rho(\mathbf{A}^p) = ||\vv(\mathbf{A}^p)|| = \left(\sum_i \lambda_i^{2p}\right)^{1/2}$.
\item $\sum_{\mathbf{x} \in {\cal X}}C(\mathbf{x}) = \texttt{Trace}(\mathbf{A}^{1+p}/||\mathbf{A}^p||) = \left(\sum_i \lambda_i^{1+p}\right)/\rho(\mathbf{A}^p)$.
\item The maximum value $M = \max_{\mathbf{x} \in {\cal X}} C(\mathbf{x}) \leq r_{\max} \lambda_1^{p}/\rho(\mathbf{A}^p)$.
\item The minimum value $m = \min_{\mathbf{x} \in {\cal X}} C(\mathbf{x}) \geq r_{\min} \lambda_d^{p}/\rho(\mathbf{A}^p)$.
\end{enumerate}
\end{proposition}
\begin{proof}
The proof is left in the supplementary material.%Complete a proof sketch.
\end{proof}

\begin{proposition}
\label{prop:var_bounds}
The variance $\sigma^2$ of the contributions $C(\mathbf{x})$ satisfies
\begin{equation}
\label{eq:bounds}
\sigma^2 \leq (M-\mu)(\mu-m) \leq \frac{(M-m)^2}{4} \leq \frac{r_{\max}^2\lambda_1^{2p}}{4\rho(\mathbf{A}^p)^2}, 
\end{equation}
where $M$ and $m$ are the maximum and minimum values defined above and $\mu$ is the mean of $C(\mathbf{x})$ given by $\sum_{\mathbf{x} \in {\cal X}}C(\mathbf{x}) / n$ where $n$ is the cardinality of ${\cal X}$. All of the above quantities can be computed from the spectrum of the matrix $\mathbf{A}$.
\end{proposition}

\begin{proof}
The proof can be obtained by a straightforward application of Popoviciu's inequality on variances~\cite{popoviciu1935equations} and a tighter variant by Bhatia and Davis~\cite{bhatia2000better}. The last inequality is obtained by setting $m=0$.
\end{proof}

The above shows that smaller values $p$ reduce an upper-bound on the variance of the contributions thereby equalizing their contributions.
The upper bound is a monotonic function of the exponent $p$ and is minimized when $p=0$ reducing all the spectrum to an identity matrix. This corresponds to whitening of the matrix $\mathbf{A}$.
However, complete whitening often leads to poor results while intermediate values such as $p=1/2$ can be significantly better than $p=1$~\cite{me_tensor_tech_rep,me_tensor,lin2017improved,secord_peihua_li}.
In the experiments section we evaluate these bounds on deep features from real data.

\begin{proposition}
\label{prop:linear_span} 
For exponents $0 < p < 1$, the matrix power $\mathbf{A}^p$ may not lie in the linear span of the outer-products of the features $\mathbf{x} \in {\cal X}$.
\end{proposition}

%\begin{proof}The proof is left in the supplementary material.\end{proof}

The proof of Proposition~\ref{prop:linear_span} is left in the supplementary material. A consequence of this is that the matrix power cannot be easily computed in the low-dimensional embedding space of outer-products encoding such as Tensor Sketch. It does however lie in the linear span of the outer-products of the eigenvectors. However, computing eigenvectors can be significantly slower than computing weighted aggregates. 
We describe the computation and memory trade-offs between computing the matrix powers and democratic pooling in Section~\ref{sec:exp_time}.

\section{Experiments}
\label{sec:experiments}
We analyze the behavior of matrix power normalization and $\gamma$-democratic pooling empirically on several fine-grained and texture recognition datasets. The general experiment setting and the datasets are described in Section~\ref{sec:exp_setup}. 
We validate the theoretical bounds on the feature contributions with real data in Section~\ref{sec:exp_bounds}. 
We compare our models against sum-pooling baseline, matrix power normalization, and other state-of-the-art methods in Sections~\ref{sec:eval} and \ref{sec:ts}.
Finally, we include a discussion on runtime and memory consumption for various approaches and a technique to perform end-to-end fine-tuning in Section~\ref{sec:exp_time}.

%In section~\ref{sec:eval} we train linear classifiers on the representations and compare the recognition accuracy.  We then show that $\gamma$-democratic pooling is computationally more efficient than matrix-square root normalization and the feature dimension can be reduced by an order with tensor sketch in section~\ref{sec:exp_time}.

\subsection{Experimental setup}
\label{sec:exp_setup}
\textbf{Datasets.} We experiment on Caltech-UCSD Birds~\cite{WelinderEtal2010}, Stanford Cars~\cite{KrauseStarkDengFei-Fei_3DRR2013} and FGVC Aircrafts~\cite{maji13fine-grained} datasets. 
Birds dataset contains 11,788 images which contain over 200 bird species. Stanford Cars dataset consists of 16.185 images across 196 categories and FGVC Aircrafts provides 10,000 images of 100 categories.
For each dataset, we use the train and test splits provided by the benchmarks and only the corresponding category labels are used during training phase. %while training even detailed annotations are provided. 
In addition to the above fine-grained classification tasks, we also analyze the performance of various approaches on the following datasets: Describable Texture Dataset (DTD)~\cite{cimpoi2014describing}, Flickr Material Dataset (FMD)~\cite{Sharan09} and MIT indoor scene dataset~\cite{quattoni_mitindoors}.
DTD consists of 5,640 images across 47 texture attributes. We report results averaged over the 10 splits provided by the dataset. FMD provides 1000 images from 10 different material categories. We randomly split half of images for training and the rest for testing for each category and report results across multiple splits. 
The MIT indoor scene dataset contains 67 indoor scene categories, each of which includes 80 images for training and 20 for testing.

%\smallskip
%\vspace{0.3cm}
%\noindent \textbf{Features.} 
\paragraph{\textbf{Features.} }
We aggregate the second-order features with $\gamma$-democratic pooling and matrix power normalization using VGG-16~\cite{simonyan14very} and ResNet101~\cite{He_2016_CVPR} networks. 
We follow the work~\cite{bilinear_finegrained} and resize input images to $448 \times 448$ and aggregate the last convolutional layer features after ReLU activations. 
For the VGG-16 network architecture, this results in  feature maps of size $28\times 28 \times 512$ (before aggregation), while for the ResNet101 architecture this results in maps of size $14 \times 14 \times 2048$.
For $\gamma$-democratic pooling, we run the modified Sinkhorn algorithm for 10 iterations with the power exponent $\tau=0.5$. 
Fully democratic pooling~\cite{murray2017interferences} and sum pooling can be implemented by setting $\gamma=0$ and $\gamma=1$, respectively. 
The aggregated features are followed by element-wise signed square-root and $\ell_2$ normalization. 
For fine-grained recognition datasets, we aggregate the VGG-16 features fine-tuned with vanilla BCNN models, while the ImageNet pretrained networks without fine-tuning are used for texture and scene datasets.

\subsection{The distribution of the spectrum and feature contributions}
\label{sec:exp_bounds}
In this section, we analyze how democratic pooling and matrix normalization effect the spectrum (set of eigenvalues) of the aggregated representation, as well as how the contributions of individual features are distributed as a function of $\gamma$ for the democratic pooling and $p$ of the matrix power normalization.

We randomly sampled 50 images from CUB and MIT indoor datasets each and plotted the spectrum (normalized to unit length) and the feature vector contributions $C(\mathbf{x})$ (Eq.~\eqref{eq:contributions}) in Figure~\ref{fig:eigs}. 
In this experiment, we use the matrix power $p=0.5$ and $\gamma=0.5$.
Figure~\ref{fig:eigs}(a) shows that the square root yields a flatter spectrum in comparison to the sum aggregation.
Democratic aggregation distributes the energy away from the top eigenvalues but has considerably sharper spectrum in comparison to the square root. 
The $\gamma$-democratic pooling interpolates between sum and fully democratic pooling.
%
%Matrix square root largely reduces the value of top eigenvalues and creates the spectrum with heavier tails as shown in Figure~\ref{fig:eigs}(a) 

Figure~\ref{fig:eigs}(b) shows the contributions of each feature $\mathbf{x}$ to the aggregate for different pooling techniques (Eq.~\eqref{eq:contributions}).
The contributions are more evenly distributed for the matrix square root in comparison to sum pooling.
Democratic pooling flattens the individual contributions the most -- we note that it is explicitly designed to have this effect. These two plots show that democratic aggregation and power normalization both achieve equalization of feature contributions.

\begin{figure*}[t]
\begin{center}
\begin{tabular}{cccc}
\includegraphics[height=0.24\linewidth]{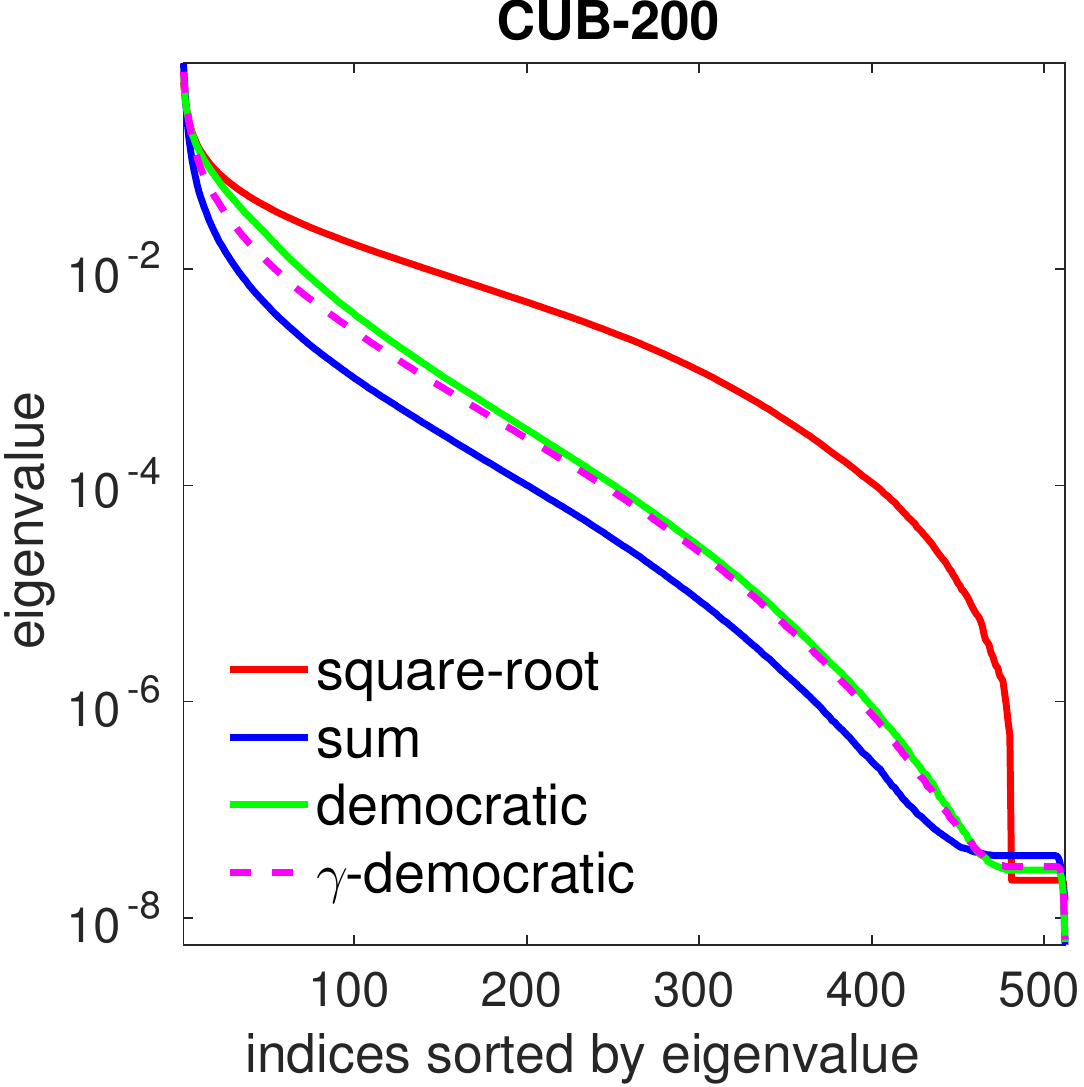} &
\includegraphics[height=0.24\linewidth]{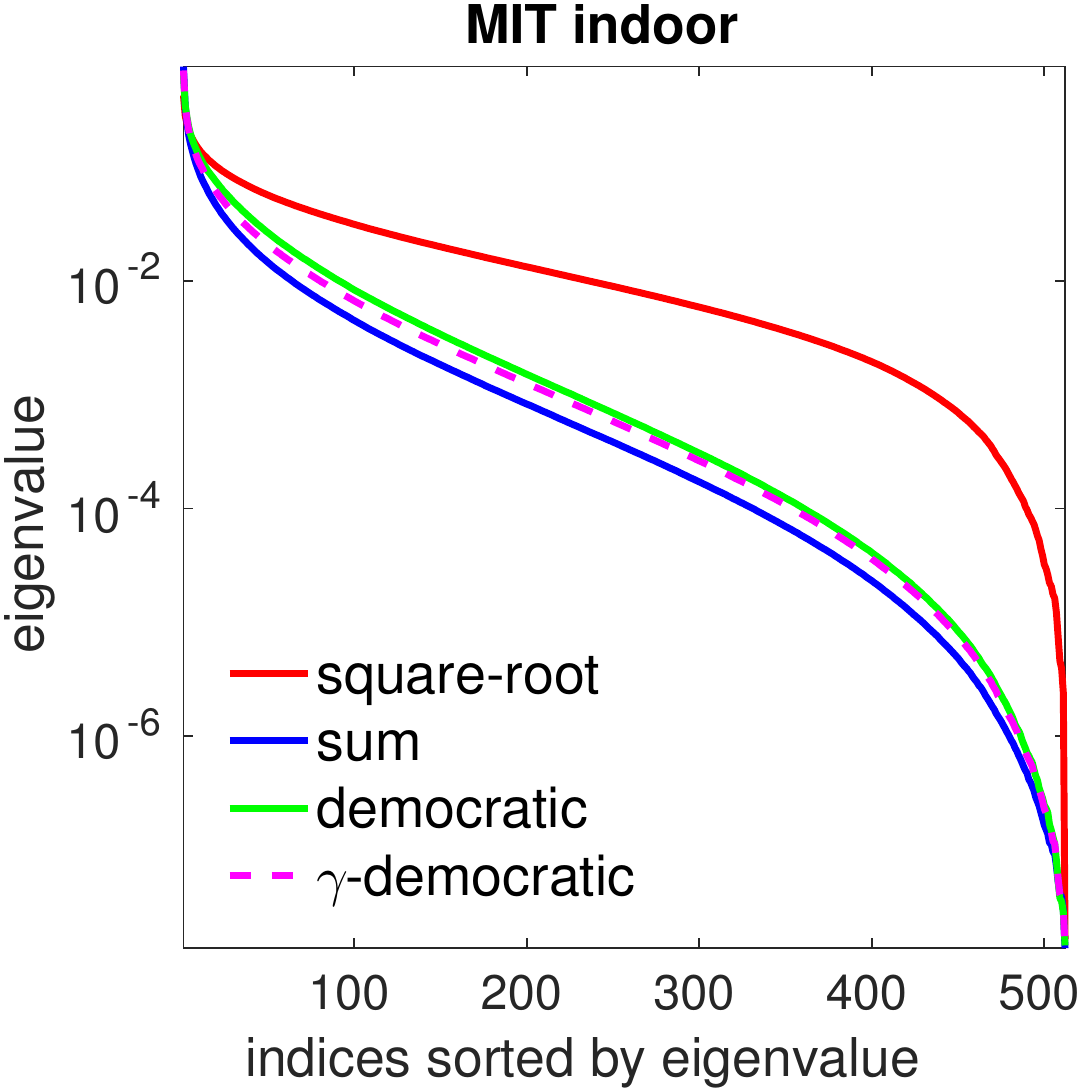} &
\includegraphics[height=0.24\linewidth]{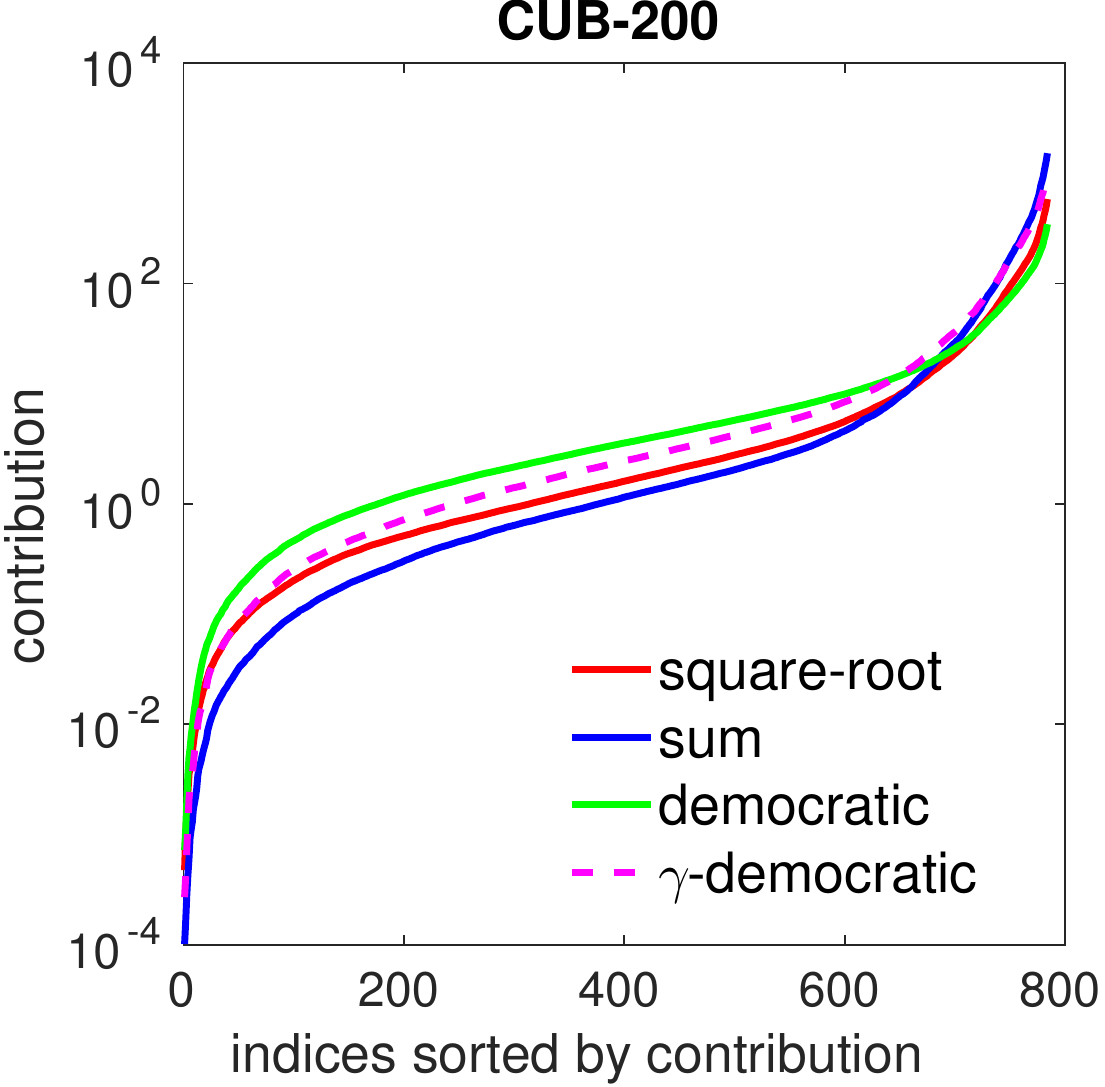} & 
\includegraphics[height=0.24\linewidth]{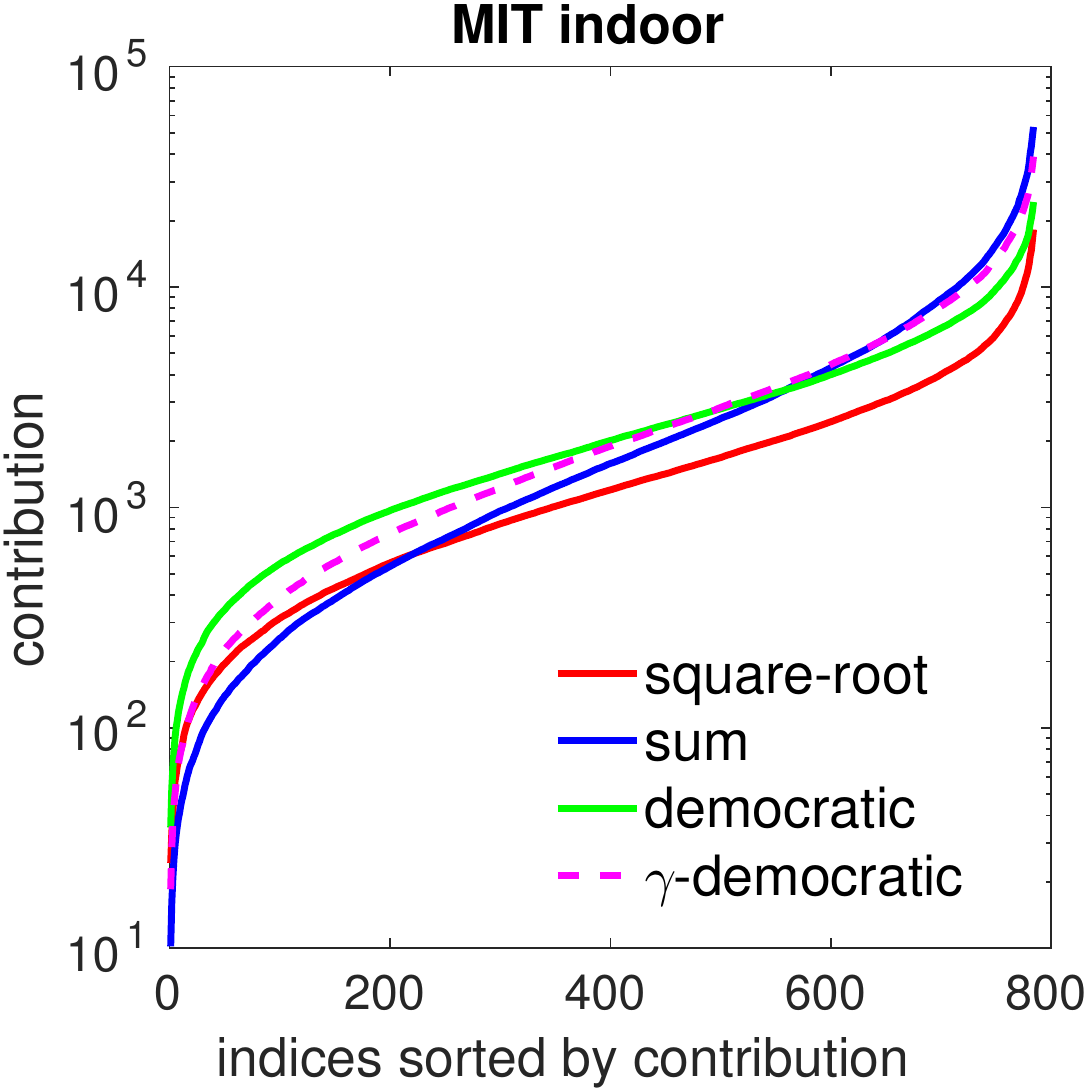} \\
\multicolumn{2}{c}{(a) spectrum (eigenvalues)} & \multicolumn{2}{c}{(b) contributions $C(\mathbf{x})$} \\
%(a) & (b)
\end{tabular}
\end{center}
%\vspace{-0.5cm}
\caption{\label{fig:eigs} (a) The spectrum (eigenvalues) for various feature aggregators on CUB-200 and MIT indoor datasets. (b) The individual feature vector contributions $C(\mathbf{x})$.}
\end{figure*}

\begin{figure*}[t]
\begin{center}
\begin{tabular}{cccc}
\includegraphics[height=0.24\linewidth]{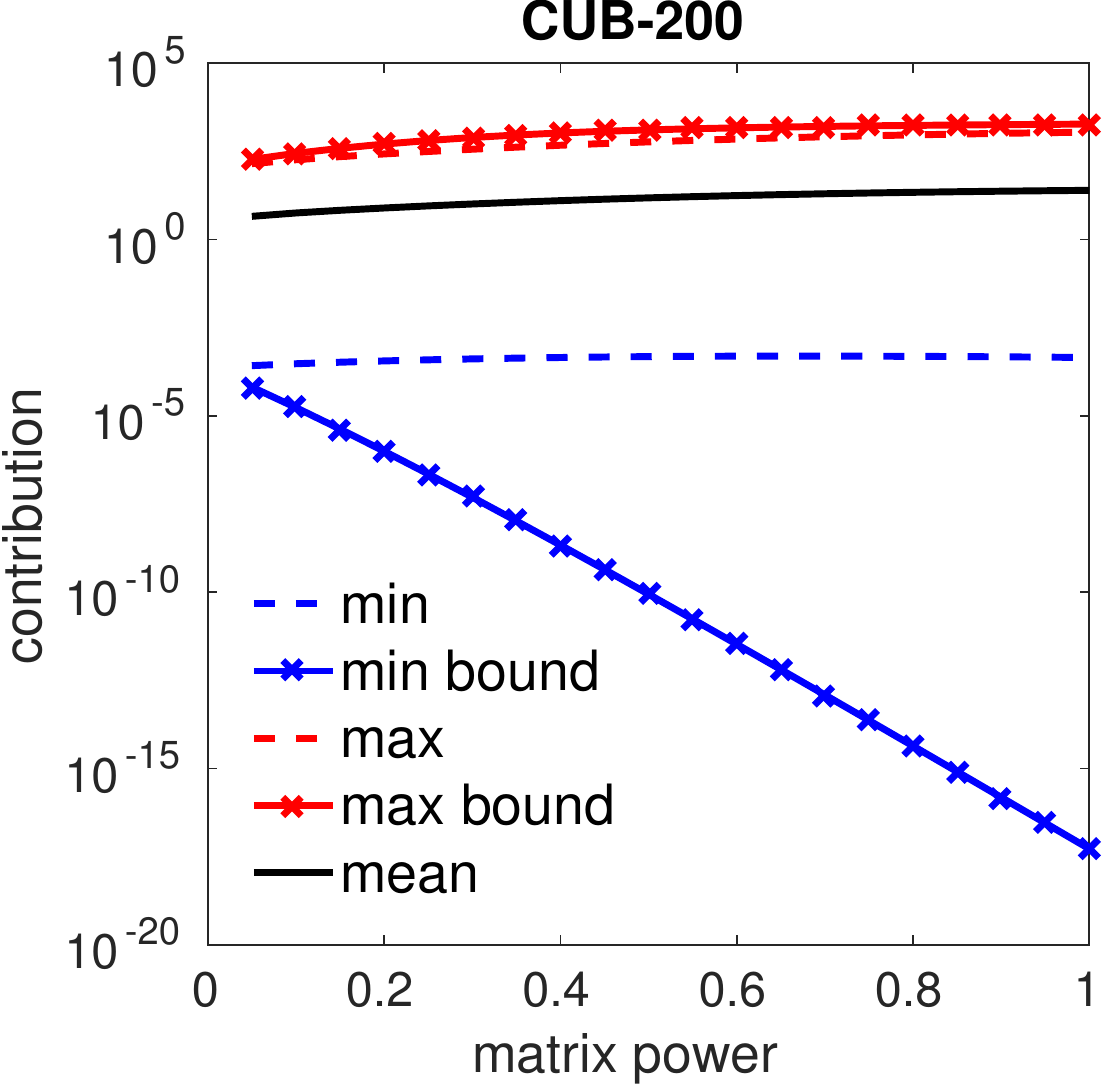} &
\includegraphics[height=0.24\linewidth]{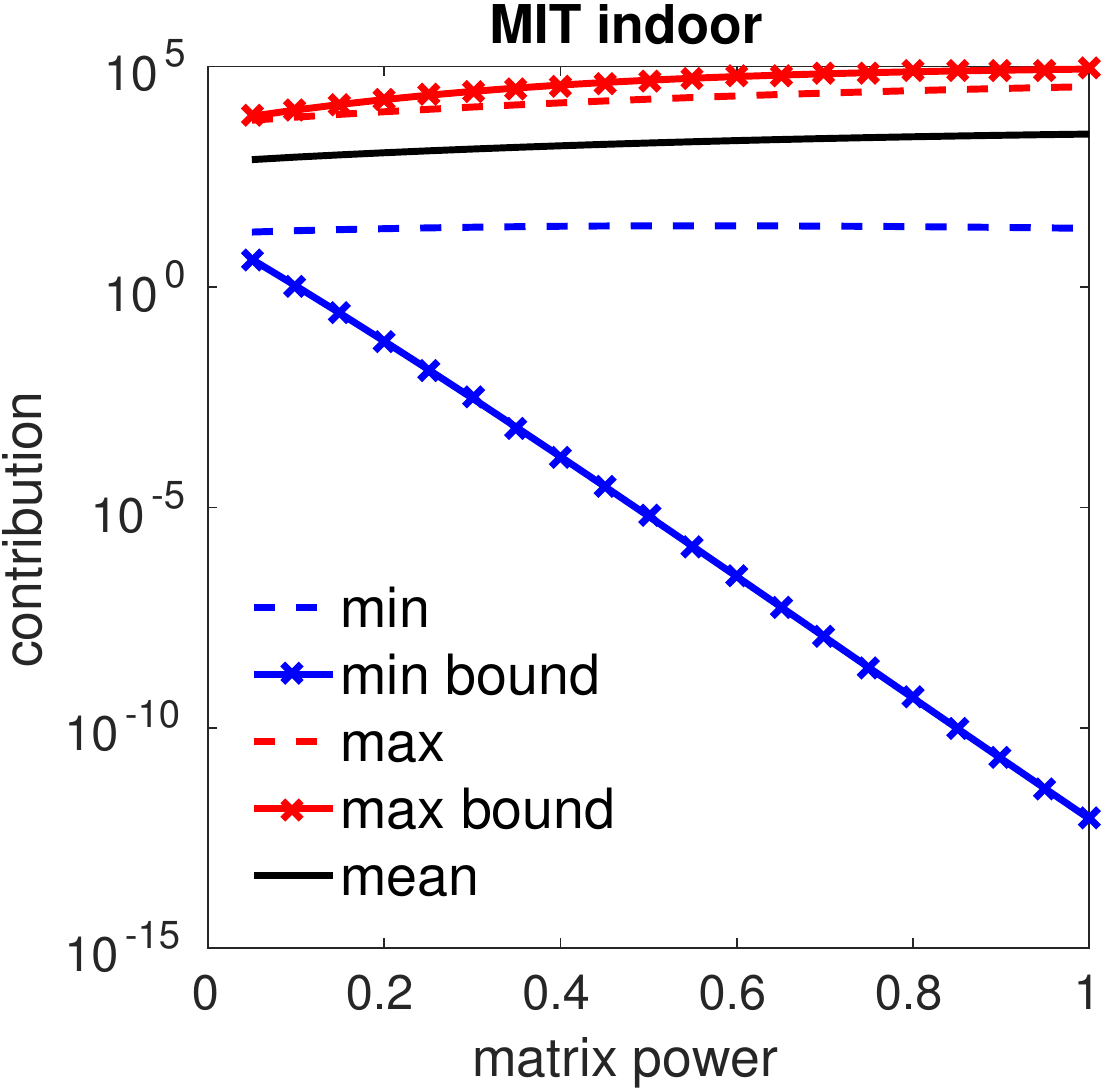} &
\includegraphics[height=0.24\linewidth]{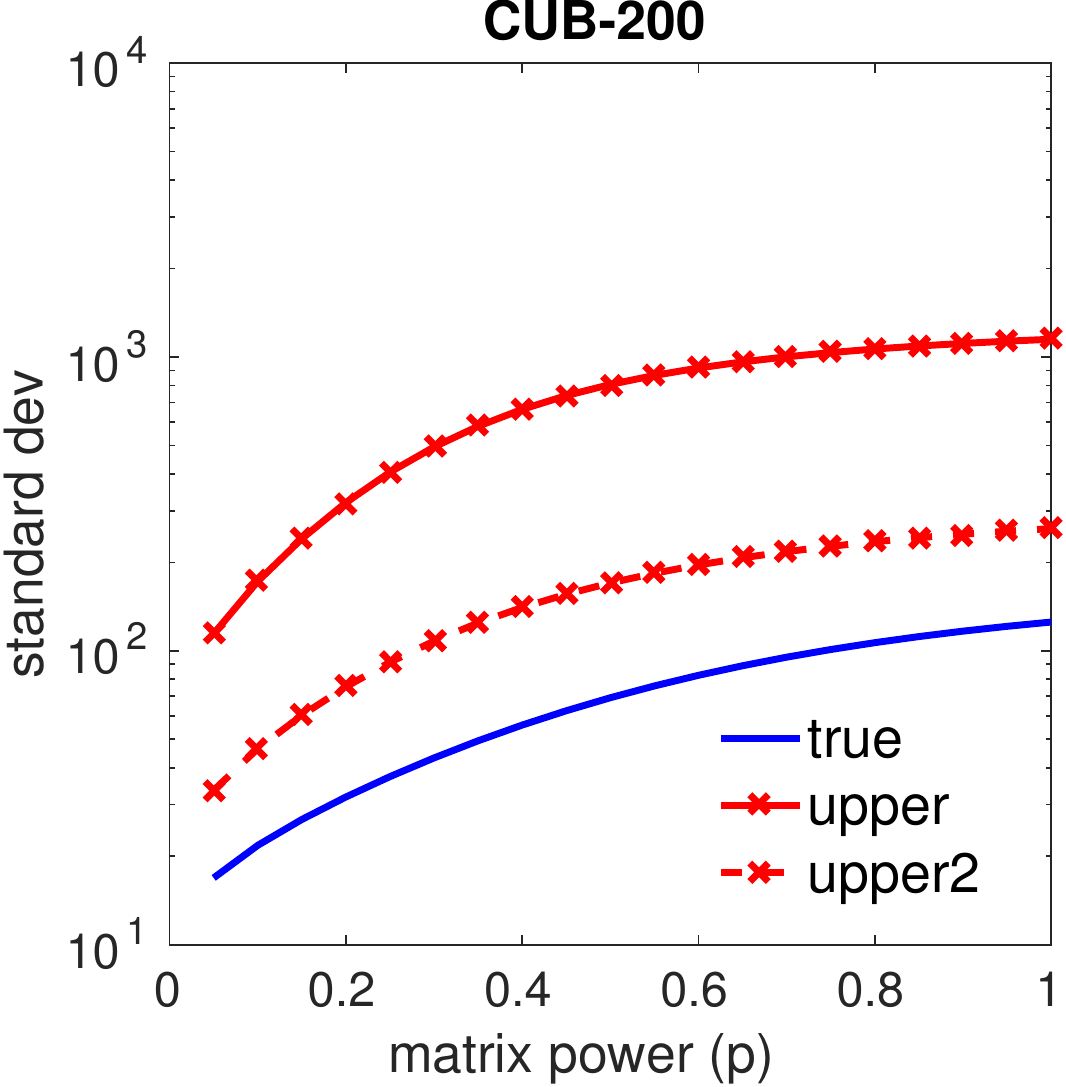} & 
\includegraphics[height=0.24\linewidth]{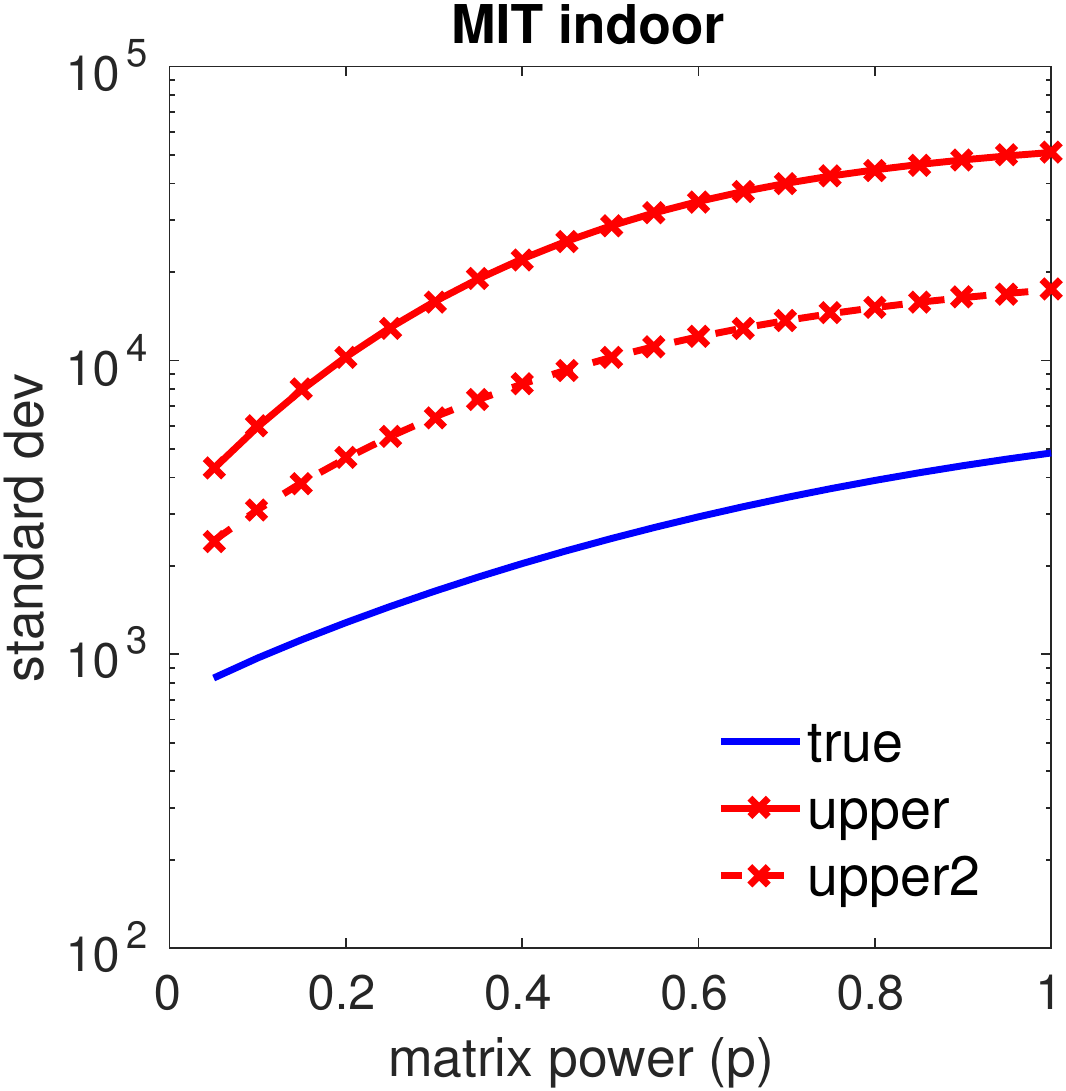} \\
\multicolumn{2}{c}{(a) bounds on contribution} & \multicolumn{2}{c}{(b) bounds on variance} \\
%(a) & (b)
\end{tabular}
\end{center}
%\vspace{-0.5cm}
\caption{\label{fig:bounds} (a) The upper (red solid) and lower bounds (blue solid) on the contributions to the set similarity versus the exponent of matrix power normalization on Birds and MIT indoor datasets. Maximum and minimum values are shown in dashed lines and the the mean is shown in black solid lines. (b) The upper bounds to the variance of feature contributions $C(\mathbf{x})$. }
\end{figure*}

Figure~\ref{fig:bounds} shows the variances of the contributions $C(\mathbf{x})$ to the aggregation $\hat{\mathbf{A}}^p$ using the VGG-16 features for different values of the exponent $p$.
Figure~\ref{fig:bounds}(a) shows the true minimum, maximum, mean as well as the bounds of these quantities expressed in Proposition~\ref{prop:four_props}. 
The upper bound on the maximum contribution, \ie, $r_{\max} \lambda_1^{p}/\rho(\mathbf{A}^p)$, is tight on both datasets, as can be seen in the overlapping red lines, while the lower bound is significantly less tight.

% the Claims 4 and 5 given the VGG-16 feature vectors. The value $r_{\max} \lambda_1^{p}/\rho(\mathbf{A}^p)$ shown in solid red curve in Figure~\ref{fig:bounds}(a) tightly bounds the maximum contributions shown in the red dashed curve--both curves closely overlap.

Figure~\ref{fig:bounds}(b) shows the true deviation and two different upper bounds on the variance of the contributions as expressed in Proposition~\ref{prop:var_bounds} and Eq.~\eqref{eq:bounds}.
The tighter bound shown by the dashed red line corresponds to the version with the mean $\mu$ in Eq.~\eqref{eq:bounds}.
The plot shows that the matrix power normalization implicitly reduces the variance in feature contributions similar to equalizing the feature vector contributions $C(\mathbf{x})$ in democratic aggregation.
These plots are averaged over 50 examples from the CUB-200 and MIT indoor datasets. % like the earlier.

%The maximum contribution (red dashed lines) is bound from above by the maximum eigenvalues (red solid lines) and the minimum (blue dashed line) is bounded from below by the minimum eigenvalues (blue  solid lines)

\subsection{Effect of $\gamma$ on democratic pooling}
\label{sec:eval}
Table~\ref{tab:acc_table} shows the performance as a function of $\gamma$ for the $\gamma$-democratic pooling and $p$ for the matrix normalization on the VGG-16 network. 
For DTD dataset, we report results on the first split. 
For FMD dataset, we randomly sample half of the data in each category for training and use the rest for testing. 
We use the standard training and testing splits on remaining datasets. 
We augment the training set by flipping its images and train k one-vs-all linear SVM classifiers with hyperparameter $C=1$. 
At the test time, we average predictions from an image and its flipped copy. 
Optimal $\gamma$ and the matrix power $p$ are also reported.

The results on sum pooling correspond to the symmetric BCNN models~\cite{lin2017improved}. 
Fully democratic pooling ($\gamma$=0) improves the performance over sum pooling by 0.7-1\%. 
However, equalizing feature contributions hurts performance on Stanford Cars and FMD dataset. 
Table~\ref{tab:acc_table} shows that reducing the contributions by adjusting $0 < \gamma < 1$ helps outperform sum pooling and fully democratic pooling.

Matrix power normalization outperforms $\gamma$-democratic pooling by 0.2-1\%. 
However, computing the matrix powers on covariance matrices is computationally expensive compared to our democratic aggregation. We discuss these tradeoffs in the Section~\ref{sec:exp_time}.
%According to Claim 6, reducing the size of representetions based on the matrix power is also hard to achieve. %This drawback limits the usage of matrix power normalization to aggregate more discriminative features which usually have a higher dimension such as ResNet features.

\begin{table}[t]
\setlength{\tabcolsep}{10pt}
\renewcommand{\arraystretch}{1.1}
\centering
\begin{tabular}{l|c|c|c|c}
\multirow{3}{*}{Dataset} & \multicolumn{3}{|c|}{$\gamma$-democratic} & \multirow{3}{*}{$\mathbf{A}^{p}$} \\
\cline{2-4}
 & Democratic & Optimal& Sum & \\
 & $\gamma$=0 & $\gamma$  & $\gamma=1$ & \\
\hline
Caltech UCSD Birds & 84.7 & 84.9 (0.5) & 84.0 & 85.9 (0.3)\\
Stanford Cars & 89.7 & 90.8 (0.5) & 90.6 & 91.7 (0.5)\\
FGVC Aircrafts & 86.7 & 86.7 (0.0) & 85.7 & 87.6 (0.3)\\
\hline
DTD & 72.2 & 72.3 (0.3) & 71.2 & 72.9 (0.6)\\
FMD & 82.8 & 84.8 (0.8) & 84.6 & 85.0 (0.7)\\
MIT indoor & 79.6 & 80.4 (0.3) & 79.5 & 80.9 (0.6)\\
\end{tabular} 
%\smallskip
%\vspace{0.1in}
\caption{\label{tab:acc_table} The accuracy of aggregating second-order features w.r.t. various aggregators using fine-tuned VGG-16 on fine-grained recognition (top) and using ImageNet pretrained VGG-16 on other (bottom) datasets. From left to right, we vary $\gamma$ values and compare democratic pooling, $\gamma$-democratic pooling and average pooling with the matrix power aggregation. The optimal values of $\gamma$ and $p$ are indicated in parentheses.}
%\vspace{-0.5cm}
\end{table}

\subsection{Democratic pooling with Tensor Sketching}
\label{sec:ts}
One of the main advantages of the democratic pooling approaches over matrix power normalization techniques is that the embeddings can be computed in a low-dimensional space using tensor sketching.
To demonstrate this advantage, we compute the second-order democratic pooling combined with tensor sketching on 2048 dimensional ResNet-101 features. 
Direct construction of second-order features yields $\sim$4M dimensional features which are impractical to manipulate on GPU/CPU. 
Therefore, we apply the Tensor Sketch~\cite{Pham_sketch} to approximate the outer product using 8192 dimensional features, which is far lower than 2048$^2$ of the full outer product.
The features are aggregated using $\gamma$-democratic approach with $\gamma=0.5$.
We compare our method to the state of the art on MIT indoor, FMD and DTD datasets. 
We report the mean accuracy. For DTD and FMD, we also indicate the standard deviation over 10 splits.

%\smallskip
%\vspace{0.1cm}
%\noindent{\textbf{Results on MIT indoor.}} 
\paragraph{\textbf{Results on MIT indoor.}} 
Table~\ref{tab:mit67} reports the accuracy %of recently proposed methods 
on MIT indoor. The baseline model approximating second-order features with tensor sketch followed by sum pooling achieves 82.8\% accuracy. 
With democratic pooling, our model achieves state-of-the-art accuracy of  84.3\% which is 1.5\% more than the baseline. 
Moreover, Table~\ref{tab:acc_table} shows that we outperform the matrix power normalization using VGG-16 network by 3.4\%. Note that (i) matrix power normalization is impractical for ResNet101 features, (ii) it cannot be computed by sketching due to Proposition~\ref{prop:linear_span}. 
We also outperform FASON~\cite{Dai_2017_CVPR} by 2.6\%. FASON fuses the first- and second-order features from ${conv4\_4}$ and $conv5\_4$ layers of the VGG-19 networks given 448$\times$448 image size and scores 81.7\% accuracy. Recent work on Spectral Features~\cite{khan2017scene} achieves the same accuracy as our best model with democratic pooling. However, approach \cite{khan2017scene} uses more data augmentations (rotation, shifts, \etc) during training and pretrains the VGG-19 network on the large-scale Places205 dataset.
In contrast, our networks are pretrained on ImageNet which arguably has a larger domain shift from the MIT indoor dataset than Places205.

\begin{table}[h]
%\vspace{-0.3cm}
\setlength{\tabcolsep}{10pt}
\centering
\begin{tabular}{l l|c}
Method && accuracy\\
\hline
%{\em CNNs with Deep Supervision} & \cite{places_mit_more} & 76.1\todo{can't find this number}\\
{\em Places-205}\kern-0.6em&\cite{places_mit} & 80.9 \\
{\em Deep Filter Banks}\kern-0.6em &\cite{cimpoi2015deep} & 81.0 \\
{\em Spectral Features}\kern-0.6em&\cite{khan2017scene} & 84.3 \\
{\em FASON}\kern-0.6em & \cite{Dai_2017_CVPR}  & 81.7 \\
\hline
%Baseline (Image Size=224) && 81.7\\
{\em ResNet101 + TS + sum pooling}  & \textit{(baseline)} & 82.8\\
{\em ResNet101 + TS + $\gamma$-democratic}  & \textit{(ours)} & \textbf{84.3}\\
\end{tabular}
%\smallskip
%\vspace{0.1in}
\caption{Evaluations and comparisons to the state of the art on MIT indoor dataset.}
%\vspace{-0.5cm}
\label{tab:mit67}
\end{table}

%\noindent{\textbf{Results on FMD. }} 
\paragraph{\textbf{Results on FMD. }} 
Table~\ref{tab:fmd} compares the accuracy on FMD dataset. Recent work on Deep filter banks~\cite{cimpoi2015deep}, denoted as FV+FC+CNN, which combines fully-connected CNN features and Fisher Vector approach, scores 82.1\% accuracy. In contrast to several methods, FASON uses single-scale input images (224$\times$224) and also scores 82.1\% accuracy. Our second-order democratic pooling outperforms FASON by 0.7\% given the same image size. For 448$\times$448 image size, our model scores 84.3\% and outperforms other state-of-the-art approaches. 

\begin{table}[h]
%\vspace{-0.3cm}
\setlength{\tabcolsep}{10pt}
\centering
\begin{tabular}{l l|c|c}
Method  && input size& accuracy \\
\hline
{\em IFV+DeCAF}\kern-0.6em & \cite{cimpoi2014describing} & ms & 65.5 ${\pm}$  1.3  \\
{\em FV+FC+CNN}\kern-0.6em & \cite{cimpoi2015deep}       & ms & 82.2 ${\pm}$ 1.4 \\
{\em LFV}\kern-0.6em & \cite{Song_2017_ICCV} & ms & 82.1 ${\pm}$ 1.9  \\
{\em SMO Task}\kern-0.6em & \cite{zhang2016integrating}  & - & 82.3 ${\pm}$ 1.7 \\
{\em FASON}\kern-0.6em & \cite{Dai_2017_CVPR}  & 224 & 82.1 ${\pm}$ 1.9 \\
\hline
%Baseline (Image Size=224) && 81.7\\
{\em ResNet101 + TS + sum pooling}  & \textit{(baseline)} & 448 & 83.7 ${\pm}$ 1.3\\
{\em ResNet101 + TS + $\gamma$-democratic}  & \textit{(ours)} & 448 & \textbf{84.3} ${\pm}$ 1.5\\
{\em ResNet101 + TS + $\gamma$-democratic}  & \textit{(ours)} & 224 & 82.8 ${\pm}$ 2.5\\
\end{tabular}
%\smallskip
%\vspace{0.1in}
\caption{Evaluations and comparisons to the state of the art on the FMD dataset. The middle column indicates the image size used by each method ({\em ms} indicates multiple scales while hyphen denotes an unknown size).}
%\vspace{-0.5cm}
\label{tab:fmd}
\end{table}

%\noindent{\textbf{Results on DTD.}} 
\paragraph{\textbf{Results on DTD.}} 
Table~\ref{tab:dtd} presents our results and comparisons on DTD dataset. Deep filter banks~\cite{cimpoi2015deep}, denoted as FV+FC+CNN, reports 75.5\% accuracy. %by combining Fisher Vector representation with fully-connected features of VGG16 network using multiple-scaled inputs. 
Combined second-order features and tensor sketching outperforms Deep filter banks by 0.3\%. With second-order democratic pooling and 448$\times$448 size images, our model achieves 76.2\% accuracy and outperforms FV+FC+CNN 0.7\%. Note that FV+FC+CNN exploits several scales of image sizes.

\begin{table}[h]
\setlength{\tabcolsep}{10pt}
\centering
\begin{tabular}{l l|c|c}
Method  && input size & accuracy \\
\hline
{\em LFV}\kern-0.6em & \cite{Song_2017_ICCV} & ms &73.8 ${\pm}$ 1.0 \\
{\em FV+FC+CNN}\kern-0.6em & \cite{cimpoi2015deep} &  ms & 75.5 ${\pm}$ 0.8\\
{\em FASON}\kern-0.6em & \cite{Dai_2017_CVPR} & 224 & 72.9 ${\pm}$ 0.7\\
\hline
%Baseline (Image Size=224) && 81.7\\
{\em ResNet101 + TS + sum pooling}  & \textit{(baseline)} & 448 & 75.8 ${\pm}$ 0.7\\
{\em ResNet101 + TS + $\gamma$-democratic}  & \textit{(ours)} & 448 & \textbf{76.2} ${\pm}$ 0.7\\
{\em ResNet101 + TS + $\gamma$-democratic}  & \textit{(ours)} & 224 & 73.0 ${\pm}$ 0.6\\
\end{tabular}
%\vspace{0.1in}
%\smallskip
\caption{Evaluations and comparisons to the state of the art on the DTD dataset. The middle column indicates the image size used by each method ({\em ms} indicates multiple scales while hyphen denotes an unknown size).}
%\vspace{-0.3cm}
\label{tab:dtd}
\end{table}

\subsection{Discussion}
\label{sec:exp_time}
While matrix power normalization achieves marginally better performance, it requires SVD which is computationally expensive and not GPU friendly \eg, the CUDA BLAS cannot perform SVD for large matrices. Even in the case of matrix square root which can be approximated via Newton's iterations~\cite{lin2017improved}, the iterations involve matrix-matrix multiplication of $\mathcal{O}(n^3)$ complexity. In contrast, solving democratic pooling via the Sinkhorn algorithm (Algorithm~\ref{alg:sinkhorn}) involves only matrix-vector multiplication which is $\mathcal{O}(n^2)$. Empirically, we find that solving Sinkhorn iterations is an order of magnitude faster than solving the matrix square root on a NVIDIA Titan X GPU. Moreover, the complexity of Sinkhorn iteration depends only on the kernel matrix -- it is independent of the feature vector size. In contrast, the memory required by a covariance matrix grows with $\mathcal{O}(n^2)$ which becomes prohibitive for feature vectors greater than 512 dimensions. Second-order democratic pooling with tensor sketching yields comparable results and reduces the memory usage by two orders of magnitude over the matrix power normalization. 

Although we did not report results using end-to-end training, one can easily obtain the gradients of the Sinkhorn algorithm using automatic differentiation by implementing Algorithm~\ref{alg:sinkhorn} in a library such as PyTorch or Tensorflow. 
Training using gradients from iterative solvers has been performed in
a number of applications (\eg,~\cite{genevay2017learning}
and~\cite{mena2018learning}) which suggests that it is a promising direction.

\section{Conclusions}
\label{sec:conclusions}
We proposed a second-order aggregation method referred to as $\gamma$-democratic pooling that interpolates between sum ($\gamma$=1) and democratic pooling ($\gamma$=0) and outperforms other aggregation approaches on several classification tasks. We demonstrated that our approach enjoys low computational complexity compared to the matrix square root approximations via Newton's iterations. With the use of sketching, our approach is not limited to aggregating small feature vectors which is typically the case for the matrix power normalization.
The source code for the project is available at \url{http://vis-www.cs.umass.edu/o2dp}.

%\noindent\textbf{Acknowlegements.} 

\paragraph{\textbf{Acknowlegements.}} 
We acknowledge support from NSF (\#1617917, \#1749833) and the MassTech Collaborative grant for funding the UMass GPU cluster.

\section*{Supplementary}
In supplementary we provide the proofs for Proposition~\ref{prop:four_props} and ~\ref{prop:linear_span} described in Section~\ref{sec:method} of the paper. In addition, we provide the comparison between aggregating first- and second-order features using democratic pooling in the last section.

\subsection*{Proofs of Proposition~\ref{prop:four_props}}

\begin{enumerate}[itemsep=1.5mm]
\item 
The $\ell_2$ norm of $\vv(\mathbf{A}^p)$ is $\rho(\mathbf{A}^p) = ||\vv(\mathbf{A}^p)|| = \left(\sum_i \lambda_i^{2p}\right)^{1/2}$.
\begin{proof}
We have:
\begin{eqnarray*}
||\vv(\mathbf{A}^p)||^2 &=& \vv(\mathbf{A}^p)^T\vv(\mathbf{A}^p) \\
&=&  \texttt{Trace}((\mathbf{A}^p)^T\mathbf{A}^p) \\
&=& \texttt{Trace}(\mathbf{U}\Lambda^{2p}\mathbf{U}^T) \\
&=& \sum_i \lambda_i^{2p}.
\end{eqnarray*}

Thus the $\ell_2$ norm: $\rho(\mathbf{A}^p) = ||\vv(\mathbf{A}^p) || =  \left(\sum_i \lambda_i^{2p}\right)^{1/2}$
\end{proof}

\item $\sum_{\mathbf{x} \in {\cal X}}C(\mathbf{x}) = \texttt{Trace}(\mathbf{A}^{1+p}/||\mathbf{A}^p||) = \left(\sum_i \lambda_i^{1+p}\right)/\rho(\mathbf{A}^p)$.

\begin{proof}
We have:
\begin{eqnarray*}
\sum_{\mathbf{x} \in {\cal X}}C(\mathbf{x}) &=& \sum_{\mathbf{x} \in {\cal X}}\vv(\mathbf{x}\mathbf{x}^T)^T\vv(\mathbf{\hat{A}}^p) \\
&=& \vv(\mathbf{A})^T\vv(\mathbf{\hat{A}}^p) \\
&=& \texttt{Trace}(\mathbf{A}^T\mathbf{\hat{A}}^p) \\
&=& \texttt{Trace}(\mathbf{A}^T\mathbf{A}^p) / \rho(\mathbf{A}^p)\\
&=& \left(\sum_i \lambda_i^{1+p}\right)/\rho(\mathbf{A}^p)
\end{eqnarray*}
\end{proof}

\item The maximum value $M = \max_{\mathbf{x} \in {\cal X}} C(\mathbf{x}) \leq r_{\max} \lambda_1^{p}/\rho(\mathbf{A}^p)$.

\begin{proof}
We have
\begin{eqnarray*}
C(\mathbf{x}) &=& \vv(\mathbf{x}\mathbf{x}^T)^T\vv(\mathbf{\hat{A}}^p) \\
&=& \texttt{Trace}((\mathbf{x}\mathbf{x}^T)^T\mathbf{\hat{A}}^p) \\
&=& \texttt{Trace}(\mathbf{x}^T\mathbf{A}^p\mathbf{x}) / \rho(\mathbf{A}^p) \\
& \leq& ||\mathbf{x}||^2\lambda_1^p / \rho(\mathbf{A}^p) \\
&\leq& r_{\max} \lambda_1^{p}/\rho(\mathbf{A}^p)
\end{eqnarray*}
\end{proof}

\item The minimum value $m = \min_{\mathbf{x} \in {\cal X}} C(\mathbf{x}) \geq r_{\min} \lambda_d^{p}/\rho(\mathbf{A}^p)$.

\begin{proof}
We have:
\begin{eqnarray*}
C(\mathbf{x}) &=& \texttt{Trace}(\mathbf{x}^T\mathbf{A}^p\mathbf{x}) / \rho(\mathbf{A}^p) \\
&\geq& ||\mathbf{x}||^2\lambda_d^p / \rho(\mathbf{A}^p) \\
&\geq& r_{\min} \lambda_d^{p}/\rho(\mathbf{A}^p)
\end{eqnarray*}
\end{proof}
\end{enumerate}

\subsection*{Proof of Proposition~\ref{prop:linear_span}}

%For exponents $0 < p < 1$ the matrix power $\mathbf{A}^p$ may not lie in the linear span of the outer-products of the features $\mathbf{x} \in {\cal X }$.

\begin{proof}
Here is an example where the matrix power $\mathbf{A}^p$ does not lie in the linear span of the outer-products of the features $\mathbf{x} \in {\cal X }$. Consider two vectors $\mathbf{x}_1 = [1~0]^T$ and $\mathbf{x}_2 = [1~1]^T$. The covariance matrix $\mathbf{A}$ formed by the two is 
\begin{eqnarray*}
\mathbf{A} &=& \mathbf{x}_1\mathbf{x}_1^T + \mathbf{x}_2\mathbf{x}_2^T \\
			&=& \left[ \begin{array}{cc} 1 & 0 \\ 0 & 0\\\end{array} \right] +
				\left[ \begin{array}{cc} 1 & 1 \\ 1 & 1\\\end{array} \right]  \\
			&=& \left[ \begin{array}{cc} 2 & 1 \\ 1 & 1\\\end{array} \right] 
\end{eqnarray*}
The square root of the matrix $\mathbf{A}$ is:
\begin{eqnarray*}
\mathbf{A}^{1/2} = \left[ \begin{array}{c c} 1.3416 & 0.4472 \\ 0.4472 & 0.8944\\\end{array} \right] 
\end{eqnarray*}
It is easy to see that $\mathbf{A}^{1/2}$ cannot be written as a linear combination of $\mathbf{x}_1\mathbf{x}_1^T$ and $\mathbf{x}_2\mathbf{x}_2^T$ since any linear combination will have all equal values for all the entries except possibly the top left value.

A sufficient condition for $\mathbf{A}^p$ to be in the linear span of outer products is that the vectors $\mathbf{x} \in {\cal X}$ which are used in constructing $\mathbf{A}$ be orthogonal to each other.
This however, is not true in general for features extracted from convolutional layers.
\end{proof}

\subsection*{Aggregating first- versus second-order features}

\begin{table}[h]
\setlength{\tabcolsep}{10pt}
\renewcommand{\arraystretch}{1.1}
\centering
\begin{tabular}{l|c|c|c}
Dataset & DTD & FMD & MIT indoor \\
\hline
First order & 72.1 ${\pm}$ 0.7 & 84.2 ${\pm}$ 1.2 & 80.4 \\
Second order &  76.2 ${\pm}$ 0.7 & 84.3 ${\pm}$ 1.5 & 84.3 \\
\end{tabular} 
\vspace{0.1in}
\caption{\label{tab:supp_table} Comparison between the $\gamma$-democratic pooling using first- and second-order features with $\gamma$=0.5. For DTD and FMD datasets, we report the mean accuracy and the standard deviation over 10 splits.}
\end{table}

%Second-order features has been demonstrated to be effective in fine-grained and texture recognition tasks and as discussed in the main paper, running the Sinkhorn algorithm with second-order features is as efficient as first-order features. In this section, we show that aggregating second-order features is more effective than aggregating first-order features using $\gamma$-democratic pooling. Table~\ref{tab:supp_table} compares the results  on first-order features with the results on second-order features with $\gamma=0.5$ on DTD, FMD and MIT indoor datasets. We follow the same protocol as Section 4.5 in the main paper aggregating the 2048 dimensional ResNet-101 features from the last convolutional layers using $448 \times 448$ image size. The second-order features are approximated by Tensor Sketching to 8192 dimensional. On FMD dataset which is a small dataset with only 10 categories, first-order features aggregation works as well as aggregating second-order features. Aggregating second-order features improves over the first-order features by a significant margin ($\sim$4\%) on DTD and MIT indoor. 

In this section, we show that aggregating second-order features is more effective than aggregating first-order features using $\gamma$-democratic pooling. Table~\ref{tab:supp_table} compares the results obtained on the first-order representations with the results on second-order features for $\gamma=0.5$ on DTD, FMD and MIT indoor datasets. We follow the same protocol as detailed in Section~\ref{sec:ts} in the main paper. Specifically, we aggregate the 2048 dimensional ResNet-101 features from the last convolutional layers using $448 \times 448$ image size. The second-order features are approximated by Tensor Sketching to obtain 8192 dimensional representation. On the small-scale FMD dataset which contains only 10 categories, the first-order aggregation works comparable to aggregating second-order features. In contrast, aggregating second-order features on the larger and more challenging DTD and MIT indoor datasets improves results over the first-order features by a significant margin (\textbf{$\sim$4\%}). Such improvements demonstrate robustness of our second-order aggregation scheme.

\bibliographystyle{splncs04}
\bibliography{egbib}
%

%\clearpage

\end{document}